\documentclass{article}

\usepackage[final,nonatbib]{neurips_2020}

\usepackage[utf8]{inputenc} 
\usepackage[T1]{fontenc}    
\usepackage{hyperref}       
\usepackage{url}            
\usepackage{booktabs}       
\usepackage{amsfonts}       
\usepackage{nicefrac}       
\usepackage{microtype}      

\usepackage{geometry}
\usepackage{amsmath, amssymb, amsthm}
\usepackage{thmtools, thm-restate}
\usepackage{mathtools}
\usepackage{enumitem}
\usepackage{xcolor}
\usepackage{diagbox}
\usepackage[font=small]{caption}
\usepackage{graphicx, subfig}
\usepackage{wrapfig}

\usepackage{multirow}
\usepackage[capitalise]{cleveref}
\usepackage[math]{cellspace}
\usepackage{selectp}

\def\R{\mathbb{R}}

\def\A{\mathcal{A}}
\def\B{\mathcal{B}}
\def\PP{\mathcal{P}}
\def\D{\mathcal{D}}
\def\X{\mathcal{X}}

\DeclareMathOperator*{\Argmax}{Arg\,max}
\DeclareMathOperator*{\argmax}{arg\,max}

\DeclareMathOperator*{\E}{\mathbb{E}}

\newcommand{\ambar}[1]{{\color{blue} AP: #1}}

\newcommand{\rene}[1]{\marginpar{\color{red}\vspace{-5pt}\scriptsize RV: #1}}
\newcommand{\myparagraph}[1]{\textbf{#1.}}

\newcommand{\switch}[1]{{\color{red} #1}}
\renewcommand{\switch}[1]{\iffalse#1\fi}

\newtheorem{definition}{Definition}
\newtheorem{theorem}{Theorem}

\newtheorem{manualcorollaryinner}{Corollary}
\newenvironment{manualcorollary}[1]{%
  \renewcommand\themanualcorollaryinner{#1}%
  \manualcorollaryinner
}{\endmanualcorollaryinner}

\title{A Game Theoretic Analysis of Additive \\ Adversarial Attacks and Defenses}
\author{
Ambar Pal\\
Mathematical Institute for Data Science\\
Johns Hopkins University\\
\texttt{ambar@jhu.edu} \\
\And
René Vidal\\
Mathematical Institute for Data Science\\
Johns Hopkins University\\
\texttt{rvidal@jhu.edu} \\
}

\begin{document}

\maketitle

\begin{abstract}
Research in adversarial learning follows a cat and mouse game between attackers and defenders where attacks are proposed, they are mitigated by new defenses, and subsequently new attacks are proposed that break earlier defenses, and so on. However, it has remained unclear as to whether there are conditions under which no better attacks or defenses can be proposed. In this paper, we propose a game-theoretic framework for studying attacks and defenses which exist in equilibrium. Under a locally linear decision boundary model for the underlying binary classifier, we prove that the Fast Gradient Method attack and a Randomized Smoothing defense form a Nash Equilibrium. We then show how this equilibrium defense can be approximated given finitely many samples from a data-generating distribution, and derive a generalization bound for the performance of our approximation.
\end{abstract}

\section{Introduction} Neural network classifiers have been shown to be vulnerable to additive perturbations to the input, which can cause an anomalous change in the classification output. There are several \emph{attack} methods to compute such perturbations for any input instance which assume access to the model gradient information, e.g., Fast Gradient Sign Method \cite{Goodfellow:ICLR15} and Projected Gradient Method \cite{Madry:ICLR18}. In response to such additive attacks, researchers have proposed many additive \emph{defense} methods with varying levels of success, e.g., Randomized Smoothing \cite{Cohen:ICML19}.
However, it has been later discovered that a lot of these defenses are in turn susceptible to further additive attacks handcrafted for the particular defenses. 
This back and forth where attacks are proposed breaking previous defenses and then further defenses are proposed mitigating earlier attacks has been going on for some time in the community and there are several open questions to be answered. Can all defenses be broken, or do there exist defenses for which we can get provable guarantees? Similarly, does there always exist a defense against any attack, or are there attacks with provable performance degradation guarantees? Do there exist scenarios under which attackers always win, and similarly scenarios where defenders always win? Are there conditions under which an equilibrium between attacks and defenses exist?

In this work, we answer some of these questions in the affirmative in a binary classification setting with locally linear decision boundaries. Specifically, we find a pair of attack $A$ and defense $D$ such that if the attacker uses $A$, then there is no defense which can perform better than $D$, and vice versa. Our approach can be seen as a novel way to obtain both provable attacks and provable defenses, complementing recent advances on certifiable defenses in the literature (\cite{Wong:Arxiv17,Carlini:Arxiv17,Huang:ICCAIV17,Dutta:Arxiv17,Tsuzuku:NIPS18,Hein:NIPS17,Lecuyer:SP19,Li:18}).

To summarize, our contributions are as follows:
\begin{enumerate}[leftmargin=*]
\item We introduce a game-theoretic framework for studying equilibria of attacks and defenses on an underlying binary classifier.\switch{\rene{We don't do NNs in particular. Shall we just say binary classifier? \ambar{Done}}} In order to do so, we first specify the capabilities, i.e., the changes the attacker and defender are allowed to make to the input, the amount of knowledge the attacker and the defender have about each other, and formalize the strategies that they can follow.\switch{\rene{Isn't allowed changes the same as strategies they can follow? \ambar{The strategy they can follow is the method they follow to make the changes given any $x$ (which in our case is an element of $\PP_A$ or $\PP_D$). The allowed changes are at most $\epsilon$ perturbation to the input.}}} 

\item We show that the Fast Gradient Method attack and a Randomized Smoothing defense \footnote{Note that our randomized smoothing defense does \emph{not} sample from an isotropic gaussian distribution.} form a Nash Equilibrium under the assumption of a zero-sum game with locally linear decision boundary for the underlying binary classifier and full knowledge of the data-generating distribution.

\item We propose 
an optimization-based method to approximate the optimal defense given access to a finite training set of $n$ independent samples and we derive generalization bounds on the performance of this finite-sample approximation. Our bounds show that the approximation approaches the optimal defense at a fast rate of $O(\sqrt{\log n / n})$ with the number of samples $n$. \switch{\rene{$d$ is undefined. Wasn't is supposed to be $m$? \ambar{Right, $m$ is the dimension of the data points which is constant. Removed}}}
\end{enumerate}

The rest of the paper is organized as follows: In \cref{sec:setup} we describe the 
proposed game theoretic setup.
In \cref{sec:optimality} we state our main result showing the existence of an optimal attack and defense. This is followed by \cref{sec:approximation} where we propose an optimization method to approximate the optimal defense, and provide some experiments validating our methods and models. \cref{sec:approximation} presents a generalization analysis showing that the proposed approximation approaches the optimal defense at a fast rate. Finally we conclude in \cref{sec:conclusion} by putting our work into perspective with related work.

\section{A Game Theoretic Setup for Additive Adversarial Attacks and Defenses} 
\label{sec:setup}
We will denote a random variable with an upper-case letter, e.g., $X$, and a realization of a random variable with a lower-case letter, e.g., $x$. We will consider a binary classification task with data distribution $p_X$ defined over the input space $\X\subset\R^m$. \switch{\footnote{{\color{red}RV: Usually the space of inputs is one thing, a random variable in that space is another thing, and we talk about the distribution of the random variable, so if $X$ is the space the notation $p_X$ is weird to me. I left your writeup, mine is maybe more compact, but doesn't solve the issue.}{\color{blue} Changed. Random variable is $X$ now, domain is $\X$ and distribution is $p_X$}}} The true decision boundary corresponding to the discrete labels $\{-1, +1\}$ will be defined by the zero-contour of a classifier $f \colon \X \to \R$, i.e., $\{x \colon f(x) = 0\}$, and the label of each data point $x \in \X$ will be given by ${\rm sgn}(f(x))$. 

We will define a two-player, single-shot, simultaneous, zero-sum game between an attacker $A$ and a defender $D$. In this setting the attacker and defender are allowed to simultaneously make additive perturbations $a(x)$ and $d(x)$, respectively, for a given a data point $x$, i.e., both $A$ and $D$ submit their perturbation at the same time and the perturbed point is $x+a(x) + d(x)$. The size of each perturbation is limited to be at most $\epsilon$ in $\ell_2$ norm, i.e., for all $x\in \X$, $a(x), d(x) \in V$, where $ V:= \{v \colon \|v\|_2 \leq \epsilon \}$.  

The score $u_A$ assigned to $A$%
\switch{\footnote{{\color{red}A only, or $A$ and $D$? If both, then why not $u$ instead of $u_A$? Then one maximizes the utility, the other one minimizes it} {\color{blue} Separating the attacker and the defender's utilities makes it much easier in the proofs etc. to show what we are talking about. Secondly, in this particular case we have a zero-sum game so $u_A = -u_D$, but it need not be the case in general.}}}
is determined by whether the label of the data point $x$ has changed under a locally linear approximation of $f$ around $x$ after both the perturbations $a(x)$ and $d(x)$ are applied:
\begin{equation}
u_A(x, a(x), d(x)) = 	\begin{cases}
					+1 \quad \text{ if } {\rm sgn}(f_L(x)) \neq {\rm sgn}(f_L(x + a(x) + d(x))) \\
					-1 \quad \text{ otherwise. }
				\end{cases}
\end{equation}
In the above,  $f_L(x') = f(x) + \nabla f(x)^\top (x' - x)$ is a linear approximation of $f$ in the $2\epsilon$ neighbourhood of $x$. Similarly, the score $u_D$ assigned to $D$ is defined as the negative of the utility assigned to $A$, i.e., $u_D(x,a,d) = -u_A(x,a,d)$ for all $(x,a,d)$, thus making the game zero-sum. 

\begin{figure}[h]
\centering
\includegraphics[height=5cm]{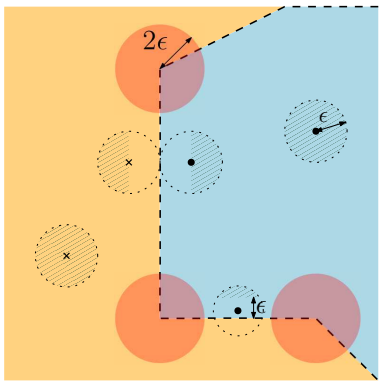}
\caption{The decision boundary is given by the dashed line. In order for our locally linear modeling assumption to hold, data-points should lie anywhere except in the red regions, i.e., within $2\epsilon$ distance to a half-plane intersection. The green shaded regions show the geometry of the robust sets $R(x)$, which follows from \cref{lem:geometry}. Observe that $R(x)$ becomes larger as $x$ moves farther away from the decision boundary.}
\label{fig:model}
\end{figure}
Note that the locally linear model assumption holds whenever the data distribution places no mass in regions of space that are less than $2\epsilon$ distance away from curved parts of the decision boundary. Concretely, for a neural network with ReLU activations, the assumption holds whenever none of the data points lie very close to the \emph{intersection} of 2 or more hyperplanes which make up the classification boundaries. This is reasonable as such regions form a set of measure zero in the input space. We refer the reader to Sec.~E of the Appendix for a more detailed discussion about the validity of such locally-flat boundary assumptions for deep neural networks.

A deterministic strategy for the attacker consists of choosing a function $a \colon \X \to V$ which dictates the perturbation $a(x)$ that is made by $A$ for the point $x \in \X$. Hence, the action space, i.e., the set of all deterministic strategies that can be followed by the attacker, is the function set $\A_A = \{ a | a \colon \X \to V \}$. 


A deterministic strategy for the defender consists of the set of all constant functions, i.e., functions which take the same perturbation direction for each point $x \in \X$. Hence the action space $\A_D$ for the defender consists of the function set $\A_D = \{d_v | v \in V, d_v \colon \X \to V \text{ s.t. } \forall x \in \X, \ d_v(x) = v \}$. Since each deterministic strategy for the defender can be uniquely with a point $v \in V$, the reader can think of $\A_D$ as $V$ for ease of understanding. \switch{\footnote{\color{red}RV: Do we really need $\A_A$ and $\A_D$, or can't we just have $\A$ and $\D$? \ambar{I felt $A$ for attacker and $\A$ for attacker's deterministic strategies seemed confusing} Also, the definition of $\A_D$ is so convoluted. Why not just $\A_D = V$? \ambar{Actually the definition stemmed from not having the restriction earlier leading to the defender always winning, and then later the constraint was added for reasons mentioned. Note added.}}} The reasons for constraining the set of strategies for the defender are twofold: first, we want to model existing literature in adversarial attacks where defenders typically follow a single strategy (e.g., smooth input, quantize input) agnostic of the test data point. Second, this restriction captures the fact that defenders are typically given the adversarially perturbed input $x  + a(x)$ and are expected to fix it without knowing the original label or original data point $x$ (and hence have to use a strategy that is agnostic to the relation between the data point and its label). 

In reality, attackers and defenders can choose even randomized strategies, where the function $a$ or $d$ is sampled according to some probability density on the set of allowed deterministic strategies. Such a randomized strategy $s_A$ for the attacker is specified by a density\switch{\rene{Probably $p_A$, not $p_a$, no?. \ambar{You're right, it looks better. Changed.}}}  $p_A \in \PP(\A_A)$,\switch{\footnote{\color{red}RV: for $X$ you used $p_X$. Now you use $p_A \in \PP$. Was it $p_X \in \PP(X)$ before? \ambar{Yes} Also, by distribution you mean pdf, cdf? \ambar{pdf. Distribution changed to density everywhere.} Finally, since the attacker chooses a random function, $\PP(\A_A)$ is the distribution of a process (like a Gaussian process), right? \ambar{Yes}}} where we define $\PP( \A_A)$ to be the set of all probability densities over $\A_A$. Similarly, a randomized strategy $s_D$ for the defender is specified by a density $p_D \in \PP(\A_D)$.\switch{\rene{$p_D$? \ambar{Changed}}} Given a pair of strategies $(s_A, s_D)$, we define the utility function of the attacker, $\bar u_A \colon \PP(\A_A) \times \PP(\A_D) \to \R$ as follows:\switch{\rene{Should quantities be random variables here, like $a(X)$? \ambar{No, everything is realization here}}}
\begin{equation}
\bar u_A (s_A, s_D) = \E_{x \sim p_X, a \sim p_A, d \sim p_D} u_A(x, a(x), d(x)) 
\end{equation}
Note that since the game is zero-sum, the corresponding utility function $\bar u_D$ for the defender is just the negative of that for the attacker. To complete the setup, we specify that both the attacker and defender have \emph{perfect knowledge} about the (possibly randomized) strategy that the other follows, and have access to the linear approximation $f_L$ around any data point $x$. In summary, we have a white-box evasion attack scenario and aim to analyze preprocessing attacks and defenses (i.e. they preprocess the input before they are fed to the classifier) that make additive perturbations to the input. 

In the following section, we will further assume both $A$ and $D$ have access to the full data-generating distribution $p_X$ in order to derive the optimal attack and defense strategy. We will show later that this assumption is not needed in practice by constructing an approximation given finitely many samples from $p_X$. We will then prove that this approximation approaches the true optimum at a fast rate. 

\section{A Characterization of Optimal Attack and Defense Strategies} \label{sec:optimality}

In this section, we will show in \cref{th:fgsm-smoothing-nash-cont} that the FGM attack and Randomized Smoothing defense exist in a Nash Equilibrium, i.e., if $A$ follows the strategy defined by the FGM attack then $D$ cannot do better than following a randomized smoothing strategy, and vice-versa. In order to show this, we will first establish in \cref{lem:fgsm-br-cont} that the FGM attack is the best response to any defense in our setting, and then show in \cref{lem:smoothing-br-cont} that randomized smoothing is a best response to FGM. Together, these lemmas would lead to our main result. Before stating our results, we begin by defining the robust set.

\begin{definition}[\bf Robust set] \label{def:robust-set}
The robust set $R(x) = \{ v \colon v \in V \text{ s.t. } \forall v' \in V \ u_A(x, v', v) = -1\}$ of a point $x\in\X$ is the subset of allowed perturbations $v \in V$ such that if the defender plays $v$ at $x$, then the attacker always gets an utility of $-1$, i.e., the least possible utility, no matter what she plays.  
\end{definition}

\cref{lem:geometry} shows that $R(x)$ is the intersection of a half-plane with $V$, as illustrated by the green shaded regions in \cref{fig:model}. Observe that when $x$ is far from the decision boundary, the robust set is equal to $V$.

\begin{restatable}[\bf Geometry of the robust set]{lemma}{geometry} \label{lem:geometry}
For any $x \in \X$, the robust set $R(x)$ is given by
\begin{equation}
R(x) = \{v \colon \text{sgn}(f(x))(f(x) + \nabla f(x)^\top v) - \epsilon \|\nabla f(x)\| \geq 0 \} \cap \{v \colon \|v\|_2 \leq \epsilon\}.
\end{equation}
\end{restatable}
\switch{\footnote{\color{red}RV: I wonder if this entire paragraph shouldn't be in the previous section. A figure accompanying the definition of $R(x)$ would also help. Alternatively, if it is not part of the setup, or of the proof of Lemma 1, then it may just make sense to define it when used, right before $\phi(v)$. \ambar{It's not part of the setup, so probably should not be in the previous section. The proof of Lemma 1 needs $R(x)$, so it should be before Lemma 1. I moved it a bit so as to not break the flow once we go into the attack description. Added a figure in the appendix, \cref{fig:model}}}}
Interestingly, the proof of the lemma shows that for a fixed $x$, in order for $v$ to achieve $u_A(x,v',v)=-1$ for all $v'\in V$, it is sufficient to ensure that $u_A(x,v',v)=-1$ when $v'$ is chosen according to the Fast Gradient Method (FGM). The FGM attack was proposed in \cite{Goodfellow:ICLR15} and makes the additive perturbation $a_\text{FGM}(x) = -\epsilon \frac{{\rm sgn} (f(x))}{\| \nabla f(x) \|_2} \nabla f (x)$. Note that the original attack was called Fast Gradient Sign Method, as it was derived for $\ell_\infty$ bounded perturbations. The same attack for $\ell_2$ bounded perturbations is called the FGM attack. Since this attack does not involve any randomness, the strategy $s_{\rm FGM}$ followed by the attacker in our framework places probability $1$ on the function $a_{\rm FGM}$.


\begin{restatable}[\bf FGM is a best-response to any defense]{lemma}{fgsm}
\label{lem:fgsm-br-cont}
For any strategy $s_D \in \PP(\A_D)$ played by the defender $D$, the strategy $s_\text{FGM} \in \PP(\A_A)$ played by the attacker $A$ achieves the largest possible utility against $s_D$, i.e., $\bar u_A(s_\text{FGM}, s_D) \geq \bar u_A(s_A, s_D)$ for all $s_A \in \PP(\A_A)$. \switch{\footnote{\color{red}RV: why $s'$ as opposed to $s_A$? \ambar{Changed}}}
\end{restatable}

Having established the optimality of the FGM attack, we now turn our attention to finding an optimal defense strategy. To that end, we define an instance of the Randomized Smoothing defense proposed in \cite{Cohen:ICML19} that will be shown to be a best-response to FGM. For any perturbation $v \in V$, we define $\phi(v)$ to be the measure of the 
set of points in $\X$ whose robust sets cointain $v$:\switch{\footnote{\color{red}RV: Is it $p(x)$ or $p_X(x)$? Is it $S(x)$ or $R(x)$? Also, the statement that precedes the equation is very confusing. Do you mean ``we define $\phi(v)$ as the probability that $v$ is contained in a robust set" or ``we define $\phi(v)$ as the probability that a robust set contains $v$" or ``we define $\phi(v)$ as the measure of the subset of $X$ such that the robust set of a point in this subset contains $v$." \ambar{Aren't all the statements equivalent? But yes, specifically I mean the last statement. Changed.}}}
\begin{equation}
\phi(v) = \int_{\X} \mathbf{1}[R(x) \ni v] p_X(x) dx 
\label{eq:phi}
\end{equation}
To define $s_\text{SMOOTH}$ we need to specify the probability distribution $p_\text{SMOOTH} \in \PP(\A_D)$. The idea is to sample uniformly from the set of maximizers of $\phi$, i.e., $V^* = \{ v^* : \phi(v^*) \geq \phi(v), ~~ \forall v\in V\}$. \switch{\footnote{\color{red}RV: Is $V^* = \{ v^* : \phi(v^*) \geq \phi(v), ~~ \forall v\in V\}$. Can't we just say it is the set of maximizers of $\phi$? Can we call it the set of most robust directions? Or the set of safest directions? \ambar{Set of maximizers is much clearer. Changed.}}} Accordingly, our defense strategy $s_\text{SMOOTH}$ samples from a uniform distribution over the set of functions $F^* = \{d_{v^*} \colon v^* \in V^*\} \subseteq \A_D$, where recall $d_v \colon \X \to V$ is the constant function defined as $\forall x \in \X \ d_v(x) = v$. 

\begin{restatable}[\bf Randomized Smoothing is a best-response to FGM]{lemma}{smooth} \label{lem:smoothing-br-cont}
The strategy $s_\text{SMOOTH}$ achieves the largest possible utility for the defender against the attack $s_\text{FGM}$ played by the attacker, i.e., for any defense $s_D \in \PP(\A_D)$ we have $\bar u_D(s_\text{FGM}, s_\text{SMOOTH}) \geq \bar u_D(s_\text{FGM}, s_D)$.
\end{restatable}

Note that \cref{lem:fgsm-br-cont} is a stronger result than we need for our further analysis, and we will only be using the following implication of \cref{lem:fgsm-br-cont}:\switch{\footnote{\color{red}Why not state as a corollary? \ambar{Done}}}
\begin{manualcorollary}{2.1}[\bf FGM is a best-response to Randomized Smoothing] \label{cor:fgsm-br-cont}
The strategy $s_\text{FGM} \in \PP(\A_A)$ played by the attacker $A$ achieves the largest possible utility against $s_\text{SMOOTH}$, i.e., $\bar u_A(s_\text{FGM}, s_\text{SMOOTH}) \geq \bar u_A(s_A, s_\text{SMOOTH})$ for all $s_A \in \PP(\A_A)$.
\end{manualcorollary}

As a consequence of \cref{cor:fgsm-br-cont} and \cref{lem:smoothing-br-cont} we have established our main result, as follows:
\begin{theorem}[\bf (FGM, Randomized Smoothing) form a Nash Equilibrium] \label{th:fgsm-smoothing-nash-cont}
Neither player gains utility by unilaterally deviating when $A$ plays $s_\text{FGM}$ and $D$ plays $s_\text{SMOOTH}$, i.e., $\forall s_A\in\PP(\A_A), s_D \in \PP(\A_D)$, we have $\bar u_A(s_\text{FGM}, s_\text{SMOOTH}) \geq \bar u_A(s_A, s_\text{SMOOTH})$ and $\bar u_D(s_\text{FGM}, s_\text{SMOOTH}) \geq \bar u_D(s_\text{FGM}, s_D)$.
\end{theorem}
\switch{{\color{blue} [AP: The next paragraph lists two important broad implications of this work, but at the same time also gives away two important directions \emph{we} are working on. I am unsure if it should be included.]}}

\myparagraph{Implications} At this point we will pause to note some implications of \cref{th:fgsm-smoothing-nash-cont}. 
\begin{enumerate}[leftmargin=*]
\item \emph{Theoretical insight:}
First, \cref{th:fgsm-smoothing-nash-cont} gives us a new theoretical insight into randomized-smoothing. Specifically, one should select the smoothing distribution according to the classifier $f$ to obtain an optimal defense, instead of sampling from the rotationally symmetric Gaussian distribution, which completely ignores the effect of the classifier. 

\item \emph{Provable attacks:} 
Second, \cref{th:fgsm-smoothing-nash-cont} shows that in some settings 
some attacks are \emph{optimal} in the sense that they
will perform better than any alternative regardless of the defense that is employed. This motivates the study of \emph{provable attacks}, something that has been largely ignored by the community which focusses a lot on \emph{provable defenses}. 

\item \emph{Winner takes all:} Third, we see from the proofs of \cref{lem:fgsm-br-cont,lem:smoothing-br-cont} that the equilibrium utility obtained by the equilibrium attacker is $1 - 2\phi(v^*)$, which shows that whenever the classification boundaries are such that $\phi(v^*) = 0$, the attacker will always win, i.e., obtain a utility of $1$ over the entire dataset. Similarly, the utility obtained by the equilibrium defender is $2\phi(v^*) - 1$, meaning that the defender wins completely whenever the classification boundaries are such that $\phi(v^*) = 1$. Relating this to the widely used metric \emph{robust accuracy}, the above is a characterization of cases where the robust accuracy obtained by the best possible defense is $0\%$ and $100\%$ respectively.
\end{enumerate}

As we saw in this section, we need access to the full data-generating distribution $p_X$ in order to compute $s_\text{smooth}$, which is an unreasonable assumption in practice. Hence, in the following section we will demonstrate how one can approximate $s_\text{smooth}$ given access to finitely many samples from $p_X$. 

\section{Approximation Properties: How to compute the optimal defense?} \label{sec:approximation}
As we saw in \cref{sec:optimality}, the optimal defense relies on the knowledge of the subset of perturbation directions that maximize $\phi$, $V^*$, which in turns depends on the distribution $p_X$ of the input data. Given $n$ i.i.d. samples $X_1, X_2, \ldots, X_n \sim p_X$, we define the finite-sample approximation of $\phi$ as:
\begin{align}
\phi_n(v) = \frac{1}{n} \sum_{i = 1}^{n} \mathbf{1}[v \in R(X_i)].\switch{\footnote{\color{red}Should $x$ be $X_i$? \ambar{Yes, fixed}}}
\end{align}
A straightforward modification of the proof of \cref{lem:smoothing-br-cont} shows that the same result can be obtained even if $p_\text{smooth}$ places all its mass on a single element of $V^*$. That is, an optimal defense can also be achieved by a deterministic strategy. Hence, our goal will now be to solve the following problem:
\begin{equation}
v_n^* = \max_v \phi_n(v) \text{ subject to } \|v\|_2 \leq \epsilon . 
\label{eq:appx}
\end{equation}

Recall from \cref{lem:geometry} that the robust set at $x_i$ can be written as $R(x_i) = \{v \in V : c_i^\top v + b_i \geq 0\}$, where $c_i = \text{sgn}(f(x_i)) \nabla f(x_i)$ and $b_i = |f(x_i)| - \epsilon \|\nabla f(x_i)\|$. Therefore, we obtain the following equivalent version of \cref{eq:appx}:
\begin{align}
\max_v \frac{1}{n} \sum_{i = 1}^{n} \mathbf{1}[c_i^\top v + b_i \geq 0] \text{ subject to } \|v\|_2 \leq \epsilon .\label{eq:appx2}
\end{align}

Observe that the objective in \eqref{eq:appx2} takes values in $\{0, \tfrac{1}{n}, \tfrac{2}{n}, \dots, 1\}$ and it is equal to $1$ iff there is a $v$ with $\|v\|\leq \epsilon$ such that $c_i^\top v + b_i \geq 0$ for all $i=1,\dots, n$. Trivially, this happens if $b_i \geq 0$ for all $i=1,\dots, n$, in which case we can choose $v=0$, meaning that no defense is needed. By inspection, we see that $b_i\geq 0$ when $|f(x_i)|$ is large, meaning that the classifier is confident about its prediction, and $\|\nabla f(x_i)\|$ is small, meaning that the response of the classifier is not very sensitive to input perturbations. This is very consistent with our intuition that defenses are needed when the classifier is not very confident (small $|f(x_i)|$) or its response is sensitive to input perturbations (large $\|\nabla f(x_i)\|$). 


To solve the optimization problem in \eqref{eq:appx2}, consider the case where there is a single sample, i.e., $n = 1$. 
In this case, if the half-space $\mathcal{H} = \{v \colon c_1^\top v + b_1 \geq 0\}$ does not intersect the hypersphere~$B(0,\epsilon)$,~then any $v^*\in V$ is a solution to \eqref{eq:appx2} with $\phi(v^*) = 0$. Else, if $\mathcal{H}$ intersects the hypersphere, a solution is given by the projection of the origin onto $\mathcal{H}$, which gives $\phi(v^*) = 1$ (see left panel of \cref{fig:regions} for illustration). When $n = 2$, there are up to two hyperplanes, which  divide the space into at most 4 regions. When no half-space intersects the hypersphere, any $v^*\in V$ is an optimal solution and $\phi(v^*)=0$. When only one half-space intersects the hypersphere, as before an optimal solution is given by the projection of the origin onto the half-space, which gives $\phi(v^*) = 1/2$. When both half-spaces intersect the hypersphere, but the half-spaces intersect each other outside the hypersphere, $v^*$ can be the projection of the origin onto either half-space. It is only when both half-spaces intersect inside the hypersphere we have $\phi(v^*)=1$ and a maximizer is given by the projection of the origin onto the intersection of both half-spaces (see right panel of \cref{fig:regions} for illustration). However, as $n$ increases, the number of regions grows exponentially in $n$, rendering such a direct region-enumeration intractable. We thus follow an optimization-based approach to find an approximate maximizer of \eqref{eq:appx2}. More specifically, we use projected gradient descent on $v$ with the constraint set $\|v\|_2 \leq \epsilon$ to solve the optimization problem in \eqref{eq:appx2} and obtain $\widehat v_n^*$. Additionally, as the gradients of the indicator function are not very useful, we use the relaxation $\mathbf{1}[\alpha \geq 0] \geq \min ( \max (0, \alpha), 1)$ and optimize the RHS. 

\begin{table}[h]
  \caption{Mean (Variance) of the approximate accuracy computed over binary classifications tasks on MNIST and FMNIST corresponding to ${10}\choose{2}$ pairs of classes. }
  \centering
  \begin{tabular}{@{}c@{\;\;}c@{\;\;}c@{\;\;}c@{}}
    \toprule
    Attack & Defense & MNIST (\%) & FMNIST (\%)\\
    \midrule
    - & - & 99.9 (0.0) & 99.9 (0.1)\\
    FGM & - & 53.3 (10.0) & 47.4 (5.1) \\
    FGM & SMOOTH & 71.2 (14.2) & 67.4 (9.0) \\
    PGD & - & 71.9 (12.0) & 74.7 (7.3) \\
    PGD & SMOOTH & 94.0 (4.0) & 90.3 (8.5) \\
    \bottomrule
  \end{tabular}
   \label{tab:results}
\end{table}

Our experiments are conducted on the MNIST and FMNIST datasets restricted to two classes. We train a 4-layer convolutional neural network with ReLU activation functions for this binary classification task. The classification results are shown in \cref{tab:results}, from which we can draw two main conclusions: (1) If the defender uses the equilibrium defense, then the attacker gets the most reduction in  approximate accuracy \footnote{Approximate Accuracy is defined as the accuracy of the model obtained by linearizing the decision boundary in a $2\epsilon$-ball around data-points, thus satisfying our modelling assumption. A detailed description, as well as more experimental details can be found in Sec.~D of the Appendix.} when using the equilibrium attack, as using any other attack improves the performance of the defended classifier. A similar statement holds from the other side. (2) The equilibrium defense SMOOTH leads to significant gains in approximate accuracy against both FGM and PGD, in agreement with our result that SMOOTH is optimal when the decision boundaries satisfy our model.

\section{Generalization Properties of the Approximation to the Optimal Defense} \label{sec:generalization}
In the previous section, we saw how to practically approximate the optimal defense as $v_n^*$ given access to a finite number of samples drawn i.i.d. from the data distribution $p_X$. But how good is this approximation?
In this section, we derive generalization bounds showing that $\phi(v_n^*)$ approaches $\phi(v^*)$ at a fast rate w.r.t. $n$. Before proceeding, we review some necessary results from learning~theory. 

\myparagraph{Learning theory review} A function $h \colon \R^m \times \ldots \times \R^m \to \R$ is said to satisfy the \emph{bounded difference} assumption if  for all $1 \leq i \leq n$ there exists finite $c_i \in \R$ such that:\switch{\rene{You are still using $d$ for dimension. You did not change 'everywhere' \ambar{Sorry, seems that I missed this one}}}
\switch{\footnote{\color{red}RV: $d$ was defense, now it is dimension of $x$. I added $m$ as the dimension at the beginning of the paper.\ambar{Changed everywhere}}}
\begin{equation}
\sup_{x_1, x_2, \ldots, x_i, \ldots x_n, x_i' \in \R^m} |h(x_1, x_2, \ldots, x_i, \ldots, x_n) - h(x_1, x_2, \ldots, x_i', \ldots, x_n)| \leq c_i.
\end{equation}
In other words, the bounded difference assumption states that $h$ changes by at most a finite amount if any of the individual inputs are changed, while keeping all others constant. Now, let $h$ be a function satisfying the bounded difference assumption, and $X_1, \ldots, X_n \sim p_X$ be i.i.d. random variables.\switch{\rene{The input distribution is $p_X$ not $\mathcal{D}$ \ambar{Changed. You're right, generalizing is not too beneficial since we have just one application}}} Then, $h$ satisfies the following useful property called the McDiarmid's inequality, which shows that that the function values of $h$ are tightly concentrated around the mean:
\begin{equation}
\Pr\Big[|h(X_1, \ldots, X_n) - \E_{X_1, \ldots, X_n \sim p_X} h(X_1, \ldots, X_n)| > \epsilon\Big] \leq \exp\Big\{\frac{-2 \epsilon^2}{\sum_i c_i^2}\Big\}.
\end{equation}
Next we need some results from \emph{Vapnik-Chervonenkis Theory}. Let $X_1, \ldots, X_n \sim p_X$ be i.i.d. random variables each taking values in $\R^m$. Let $\B$ be a family of subsets of $\R^m$. Let $B \in \B$ \switch{\rene{$A$ is attacker. You cannot use $A$ \ambar{Changed}}}be any subset in the family. Define $\mu$ as $\mu(B) = \Pr[X_1 \in B]$. Further, given a particular realization $x_1, \ldots, x_n$ of $X_1, \ldots X_n$, define the finite-sample approximation $\mu_n$ as $\mu_n(B) = \frac{1}{n} \sum_{i = 1}^n 1[X_i \in B]$. In other words, $\mu(B)$ is the probability that a sample from $p_X$ lies in $B$, and $\mu_n(B)$ estimates this probability using $n$ samples from $\D$. Taking $h(x_1, \ldots, x_n) = \sup_{B \in \B} |\mu_n(B) - \mu(B)|$ in McDiarmid's inequality, we observe that $c_i = \frac{1}{n}$ and we get the following:
\begin{equation}
\Pr\Big[ \Big|\sup_{B \in \B} |\mu_n(B) - \mu(B)| - \E_{X_1, \ldots, X_n \sim p_X} \sup_{B \in \B} |\mu_n(B) - \mu(B)| \Big| > \epsilon\Big] \leq \exp\Big\{-2 n\epsilon^2\Big\} .
\label{mu-conc}
\end{equation}
In other words, we see that the maximum inaccuracy incurred in estimating $\Pr[X_1 \in B]$ from finite samples is tightly concentrated around the mean. The final piece we need from VC theory is an upper bound on this mean inaccuracy:
\begin{equation}
 \E_{X_1, \ldots, X_n \sim p_X} \sup_{B \in \B} \Big| \mu_n(B) - \mu(B) \Big| \leq 2 \sqrt{\frac{2\log S_\B(n)}{n}},\label{mu-ub}
\end{equation}
where $S_\B(n)$ is the shatter coefficient for the family $\B$, which is defined as follows:
\begin{align}
S_\B(n) = \sup_{x_1, x_2, \ldots, x_n \in \R^m} \Big| \Big \{\{x_1, x_2, \ldots, x_n\} \cap B \colon B \in \B \Big\} \Big|.
\end{align}
In the above, each of the terms being considered in the supremum counts the number of distinct intersections with members of $\B$. For illustration, say we are working in $\R^2$, and let $\B$ be the family of subsets generated by taking each rectangle $r$ in the plane and considering $B_r$ to be the points contained in $r$. Now given any $3$ points $\{x_1, x_2, x_3\}$ in the plane, we can find rectangles $r$ such that $\{x_1, x_2, x_3\} \cap B_r$ equals each of the 8 possible subsets $\{\}, \{x_1\}, \{x_2\}, \{x_1, x_2\}, \ldots, \{x_1, x_2, x_3\}$. This shows that $S_\B(3) \geq 8$ (which implies $S_\B(3) = 8$ as $S_\B(n) \leq 2^n$). 

In other words, the shatter coefficient at $n$ equals the largest $p$\switch{\rene{$m$ is now the dimension of the data \ambar{Changed to $p$}}} such that $n$ points can be broken into $p$ subsets by members of $\B$. Hence, $S_\B(n) \geq p$ implies there exists at least one such example of $x_1, \ldots, x_n$ that can be broken into $p$ subsets. On the other hand, $S_\B(n) < p + 1$ implies that \emph{for every} choice of $n$ points, they cannot be broken into $p + 1$ or more distinct sets by members in $\B$.

\myparagraph{Generalization bound}
We will now apply the above literature to our setup.
Going forward, we will think of $n$ as being the size of the entire training data that is available to us.
Recall that in our setting, we are given a fixed base classifier $f$. Given $f$, and a sample $x_1, \ldots, x_n$, we know the robust sets $R(x_1), R(x_2), \ldots, R(x_n)$. For each direction that can be taken by the defender, i.e., $v \in V$, we define $B_v$ to be the subset of $\X$ for which $v$ is a robust direction as:\switch{\footnote{\color{red}$X$ or $\X$? \ambar{Data space changed to $\X$ everywhere.}}}
\begin{align}
B_v = \{ x \in \X \colon v \in R(x)\}.
\end{align}
Now we can define the family of subsets $\B = \{B_v \colon v \in V\}$. With this definition, we can see that $\phi$ corresponds to $\mu$, and $\phi_n$ corresponds to $\mu_n$ as:
\begin{align}
\phi(v) &:= \mu(B_v) = \Pr[X_1 \in B_v] = \Pr[v \in R(X_1)] \\
\phi_n(v) &:= \sum_{i = 1}^n 1[v \in R(x_i)] = \mu_n(B_v) .
\end{align}
Given the base classifier $f$ and the training data, we used an optimization algorithm to obtain the maximizer of $\phi_n(v)$ in \cref{sec:approximation}. We will assume for the purposes of this section that there is no optimization error, i.e., our optimizer finds the best direction for the given training data, i.e., $v_n^* = \argmax_v \phi_n(v)$.\switch{\footnote{\color{red}$v_n^*\in \Argmax (\phi_n(v))$ \ambar{Changed}}} We are interested in the difference between $\phi(v_n^*)$ and $\phi(v^*)$ \switch{\footnote{\color{red}RV: In \cref{eq:explain} one of the two $v$'s on the right should be $v^*$, but then I am not sure about inequality. Otherwise, I think really in the first term we should substitute $v^*$ by $v$. Let's discuss this in meeting. \ambar{For the second term, define $g(v) = |\phi_n(v) - \phi(v)|$. Then Eq 17 to 18 is just saying $g(v^*_n) \leq \sup_v g(v)$}}}, i.e.:
\begin{align}
\phi(v^*) - \phi(v_n^*) 
&= \Big(\phi(v^*) - \phi_n(v_n^*)\Big) + \Big(\phi_n(v_n^*) - \phi(v_n^*)\Big) \\
&
\leq |\phi(v^*) - \phi_n(v^*)| + |\phi_n(v_n^*) - \phi(v_n^*)| \\ 
&\leq \sup_v |\phi(v) - \phi_n(v)| + \sup_v |\phi_n(v) - \phi(v)|
= 2 \sup_v |\phi(v) - \phi_n(v)|.
\end{align}%
In other words, we can get an upper bound on the quantity of interest by analysing $\sup_v |\phi(v) - \phi_n(v)|$, which is the largest inaccuracy we get due to estimating $\phi(v)$ from a finite number of samples. \eqref{mu-conc} shows that this quantity is sharply concentrated at its mean, which can be upper bounded by \eqref{mu-ub} as:
\begin{align}
 \E_{X_1, \ldots, X_n \sim p_X} \sup_{v \in V} \Big| \phi_n(v) - \phi(v) \Big| \leq 2 \sqrt{\frac{2\log S_\B(n)}{n}}.
  \label{phi-ub}
\end{align}
Hence, the problem boils down to getting an upper bound on $S_\B(n)$. For families where we can obtain a bound that is sub-exponential in $n$, we can see that the RHS converges to $0$ as $n$ becomes large. Hence, we will now upper-bound the shatter coefficient $S_\B(n)$ for our setting. 

Recall that we approximate the base classifier $f$ around each data point $x$ using a linear approximation $f_L(x') = f(x) + \nabla f(x)^\top (x' - x)$. We have seen that the robust set is the region enclosed between a half-plane and the boundary of the set $V$ (see \cref{fig:model} and \cref{lem:geometry}). 
 
\begin{figure}[h]
\centering
\includegraphics[width=\textwidth]{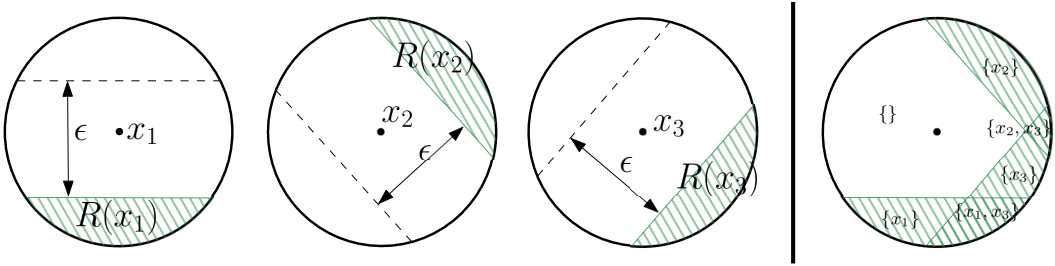}
\caption{We can compute an upper bound to $S_\B$ by looking at the different regions formed by the robust sets $R(x_i)$. The rightmost panel shows the superimposition of $R(x_1), R(x_2), R(x_3)$, showing the different subsets formed.}
\label{fig:regions}
\end{figure}

We now want to upper-bound the maximum number of different partitions of $\{x_1, x_2, \ldots, x_n\}$ that can be formed by taking subsets specified by $B_v$ for $v \in V$. Observe that overlaying all the robust sets in $V$ gives us a collection of regions, with the property that $B_v \cup \{x_1, x_2, \ldots, x_n\}$ is constant when $v$ is varied inside any region, as shown in \cref{fig:regions}.

This implies that an upper bound on $S_\B(n)$ is equal to the number of distinct regions formed. When we are in the 2-dimensional case, this is same as the number of regions formed by $n$ lines in a plane, which is known to be $(n^2 + n + 2)/2 = O(n^2)$. For higher dimensions $m$, the maximum possible number of regions grows as $O(n^m)$. Thus the upper bound given by \eqref{phi-ub} reduces to:
\begin{align}
\E_{X_1, \ldots, X_n \sim p_X} \sup_{v \in V} \Big| \phi_n(v) - \phi(v) \Big| \leq 4 \sqrt{\frac{m \log n}{n}}.
\end{align}
%
The above shows that as $n$ gets larger, i.e., we take more and more samples ($m$ is a constant), we approach the best defense at a fast rate of $O(\sqrt{\log n/n})$. This bound indicates that efficient learning from finite samples is possible. This is corroborated by our experiments, which show even faster rates. We thus believe that our generalization analysis can be sharpened by using recent advances in PAC-Bayesian learning theory, as well as modern extensions to VC-Theory. 

\section{Conclusion, Related Work and Future Directions} \label{sec:conclusion}
In this paper, we have proposed a game theoretic framework under which adversarial attacks and defenses can be studied. Under a locally linear assumption on the decision boundary of the underlying binary classifier,\switch{\rene{NN classifier, or binary classifier + locally linear boundary \ambar{Changed.}}} we identified a pair of attack and defense that exist in a Nash Equilibrium in our framework. We then gave an optimization procedure to practically compute the equilibrium defense for any given classifier, and derived generalization bounds for its performance on unseen test data. 

There has been a lot of work on the task of classification in the presence of an adversary who can make additive perturbations to the input before passing it to the classifier. There is a huge body of work on empirical attacks in the literature \cite{Miyato:ICLR16,Moosavi:SP16,Moosavi:CVPR17,Papernot:SP16,Kurakin:Arxiv16,Carlini:SP17,Papernot:Arxiv16,Chen:WAIS17,Ilyas:ICML18,Engstrom:ICML19,Su:EC19}, as well as empirically motivated work trying to mitigate the proposed attacks \cite{Miyato:ICLR16,Dziugaite:Arxiv16,Zantedeschi:WAIS17,Xu:NDSS18,Song:Arxiv17}. Here, we will focus on classical game-theoretic approaches to the problem of adversarial classification, as well as other recent defenses for which one can get theoretical guarantees on the performance under attack. 

Adversarial classification has been studied in the context of email spam detection, where we have a dataset $(\X, \mathcal{Y})$, and the adversary is allowed to modify the positive (spam) examples in the dataset (i.e., data poisoning attack) by replacing $(x, y)$ by $(x', y)$ where $y = 1$, incurring a cost of modification $c(x, x')$ according to a cost function $c$. The defender is allowed to choose a classifier $h$, which classifies any $x \in \X$ into two classes, i.e., spam or not spam. \cite{Dalvi:KDD04} studied a single shot non-zero sum game with this setup, where the defender always chooses the naïve Bayes classifier given the attacked dataset $(\X', \mathcal{Y})$. \cite{Dalvi:KDD04}  set up an integer  linear program to compute the best-response of the attacker, and give algorithms to compute the solutions efficiently. \cite{Kantarciouglu:DMKD11,Sengupta:GameSec19} have studied similar setups in a sequential setting (called a Stackelberg game), where the attacker goes first and submits the perturbed dataset $\X'$ to the defender, who then learns the classifier having observed $\X'$. \cite{Kantarciouglu:DMKD11} showed that the Stackelberg equilibrium can be approximated using optimization techniques, and provide analyses for various classification losses. \cite{Zhou:KDD12} analyzed the same setting where the classifier is now an SVM classifier, and showed that the $\min \max$ optimization problem arising from the analysis of the Nash Equilibrium of the game can be solved efficiently. \cite{Hardt:ITCS16} approached the problem from a PAC-learning perspective, showing that under certain separability assumptions on the cost function $c$ one can efficiently learn a classifier that attains low error on the attacked training set, as well as maximizes the defender's utility. 

A related line of work called adversarial hypothesis testing deals with modifications to the distribution $p_X$ from which the data is sampled instead of modifying the dataset per sample. The utility of the defender now has an additional negative term corresponding to a discrepancy function between the original $p_X$ and the modified distribution $q$. The defender now has to determine whether a given sample of $n$ points came from $p_X$ or $q$, and the defender's utility consists of a tradeoff between the Type-I and Type-II errors in this situation. \cite{Yasodharan:NIPS19} proved that mixed-strategy Nash equilibria exist in this setting, characterized them, and proved convergence properties of the classification error. \cite{Bruckner:KDD11,Bruckner:JMLR12} study Nash Equilibria for a similar setting, where the defender is now a learner who has to output the weights of a classifier that predicts which distribution the input was sampled from.

Our work differs from past literature as our defender plays an additive perturbation, instead of giving a classifier. In line with practice, we consider the base classifier fixed and provided to us to attack or defend. Additionally, we focus on the case of an attacker that can perform additive perturbations.

Finally, our work links to a recent line of work on certifiable defenses for neural networks (we refer the reader to \cite{Cohen:ICML19} for a nice review). We focus here on randomized certifiable defenses. \cite{Lecuyer:SP19} gave lower bounds on the robust accuracy of a randomized-smoothing defense via differential-privacy analyses. Subsequently, \cite{Li:NIPS19, Cohen:ICML19,Lee:NIPS19} sharpened the analysis and presented alternative techniques to obtain near-optimal smoothing guarantees for Gaussian smoothed classifiers. Our work complements these proof techniques in the literature as we use geometry of the robust sets as our primary tool to analyze the equilibria, and their optimization and generalization properties. 

There are several directions for future work. The first direction would be to extend our results to accommodate locally curved decision boundaries. We suspect that for even further extensions to base classifiers having arbitrarily complex decision boundaries, the style we use to show our generalization results would not yield useful bounds, and we would have to resort to more sophisticated tools to show convergence to the optimal defense. The second direction would be to improve the optimizer of $\phi_n$, and obtain guarantees on the optimization error. \switch{{\rene{IMHO, future work is to improve the optimizer of $\phi_n(v)$. \ambar{Added}}}}

\section*{Broader Impact}
At a high level, this work aims to provide a way to characterize adversarial attacks and defenses that might be \emph{best} for each other, in a game theoretic sense where the attacker cannot decrease the robust accuracy further when the defense is fixed, and the defender cannot increase the robust accuracy further when the attack is fixed. The technical contributions are novel geometry-flavored proof techniques that can be used to analyze provable attacks and defenses, and a game-theoretic framework to study such equilibria. Machine learning systems are increasingly being used in security-critical applications, like healthcare and automated driving: our work can be used to find guarantees on the worst accuracy a defended classifier can have under any attack. This is a step towards safe machine learning, where the ultimate goal is to be able to construct classifiers whose performance cannot be degraded by an adversary on most data-points with high probability. 

\begin{ack}
This work was supported by DARPA Grant HR00112020010 and NSF Grant 1934979. 
\end{ack}

\bibliographystyle{plain}
\bibliography{../biblio/learning}

\begin{thebibliography}{10}

\bibitem{Bruckner:JMLR12}
Michael Br{\"u}ckner, Christian Kanzow, and Tobias Scheffer.
\newblock Static prediction games for adversarial learning problems.
\newblock {\em The Journal of Machine Learning Research}, 13(1):2617--2654,
  2012.

\bibitem{Bruckner:KDD11}
Michael Br{\"u}ckner and Tobias Scheffer.
\newblock Stackelberg games for adversarial prediction problems.
\newblock In {\em Proceedings of the 17th ACM SIGKDD international conference
  on Knowledge discovery and data mining}, pages 547--555, 2011.

\bibitem{Carlini:Arxiv17}
Nicholas Carlini, Guy Katz, Clark Barrett, and David~L Dill.
\newblock Provably minimally-distorted adversarial examples.
\newblock {\em arXiv preprint arXiv:1709.10207}, 2017.

\bibitem{Carlini:SP17}
Nicholas Carlini and David Wagner.
\newblock Towards evaluating the robustness of neural networks.
\newblock In {\em 2017 ieee symposium on security and privacy (sp)}, pages
  39--57. IEEE, 2017.

\bibitem{Chen:WAIS17}
Pin-Yu Chen, Huan Zhang, Yash Sharma, Jinfeng Yi, and Cho-Jui Hsieh.
\newblock Zoo: Zeroth order optimization based black-box attacks to deep neural
  networks without training substitute models.
\newblock In {\em {ACM} Workshop on Artificial Intelligence and Security},
  pages 15--26, 2017.

\bibitem{Cohen:ICML19}
Jeremy Cohen, Elan Rosenfeld, and Zico Kolter.
\newblock Certified adversarial robustness via randomized smoothing.
\newblock {\em International Conference on Machine Learning}, 2019.

\bibitem{Dalvi:KDD04}
Nilesh Dalvi, Pedro Domingos, Sumit Sanghai, and Deepak Verma.
\newblock Adversarial classification.
\newblock In {\em {ACM SIGKDD} International Conference on Knowledge Discovery
  and Data Mining}, pages 99--108, 2004.

\bibitem{Dutta:Arxiv17}
Souradeep Dutta, Susmit Jha, Sriram Sanakaranarayanan, and Ashish Tiwari.
\newblock Output range analysis for deep neural networks.
\newblock {\em arXiv preprint arXiv:1709.09130}, 2017.

\bibitem{Dziugaite:Arxiv16}
Gintare~Karolina Dziugaite, Zoubin Ghahramani, and Daniel~M Roy.
\newblock A study of the effect of jpg compression on adversarial images.
\newblock {\em arXiv preprint arXiv:1608.00853}, 2016.

\bibitem{Engstrom:ICML19}
Logan Engstrom, Brandon Tran, Dimitris Tsipras, Ludwig Schmidt, and Aleksander
  Madry.
\newblock Exploring the landscape of spatial robustness.
\newblock {\em International Conference on Machine Learning}, 2019.

\bibitem{Goodfellow:ICLR15}
Ian~J Goodfellow, Jonathon Shlens, and Christian Szegedy.
\newblock Explaining and harnessing adversarial examples.
\newblock {\em ICLR}, 2015.

\bibitem{Hardt:ITCS16}
Moritz Hardt, Nimrod Megiddo, Christos Papadimitriou, and Mary Wootters.
\newblock Strategic classification.
\newblock In {\em Proceedings of the 2016 ACM conference on innovations in
  theoretical computer science}, pages 111--122, 2016.

\bibitem{Hein:NIPS17}
Matthias Hein and Maksym Andriushchenko.
\newblock Formal guarantees on the robustness of a classifier against
  adversarial manipulation.
\newblock In {\em Advances in Neural Information Processing Systems}, pages
  2266--2276, 2017.

\bibitem{Huang:ICCAIV17}
Xiaowei Huang, Marta Kwiatkowska, Sen Wang, and Min Wu.
\newblock Safety verification of deep neural networks.
\newblock In {\em International Conference on Computer Aided Verification},
  pages 3--29. Springer, 2017.

\bibitem{Ilyas:ICML18}
Andrew Ilyas, Logan Engstrom, Anish Athalye, and Jessy Lin.
\newblock Black-box adversarial attacks with limited queries and information.
\newblock {\em International Conference on Machine Learning}, 2018.

\bibitem{Kantarciouglu:DMKD11}
Murat Kantarc{\i}o{\u{g}}lu, Bowei Xi, and Chris Clifton.
\newblock Classifier evaluation and attribute selection against active
  adversaries.
\newblock {\em Data Mining and Knowledge Discovery}, 22(1-2):291--335, 2011.

\bibitem{Kurakin:Arxiv16}
Alexey Kurakin, Ian Goodfellow, and Samy Bengio.
\newblock Adversarial examples in the physical world.
\newblock {\em arXiv preprint arXiv:1607.02533}, 2016.

\bibitem{Lecuyer:SP19}
Mathias Lecuyer, Vaggelis Atlidakis, Roxana Geambasu, Daniel Hsu, and Suman
  Jana.
\newblock Certified robustness to adversarial examples with differential
  privacy.
\newblock In {\em {IEEE} Symposium on Security and Privacy (SP)}, pages
  656--672. IEEE, 2019.

\bibitem{Lee:Arxiv19}
Guang-He Lee, David Alvarez-Melis, and Tommi~S Jaakkola.
\newblock Towards robust, locally linear deep networks.
\newblock {\em arXiv preprint arXiv:1907.03207}, 2019.

\bibitem{Lee:NIPS19}
Guang-He Lee, Yang Yuan, Shiyu Chang, and Tommi Jaakkola.
\newblock Tight certificates of adversarial robustness for randomly smoothed
  classifiers.
\newblock In {\em Advances in Neural Information Processing Systems}, pages
  4911--4922, 2019.

\bibitem{Li:18}
Bai Li, Changyou Chen, Wenlin Wang, and Lawrence Carin.
\newblock Second-order adversarial attack and certifiable robustness.
\newblock {\em arXiv preprint: arXiv:1809.03113}, 2018.

\bibitem{Li:NIPS19}
Bai Li, Changyou Chen, Wenlin Wang, and Lawrence Carin.
\newblock Certified adversarial robustness with additive noise.
\newblock In {\em Advances in Neural Information Processing Systems}, pages
  9459--9469, 2019.

\bibitem{Madry:ICLR18}
Aleksander Madry, Aleksandar Makelov, Ludwig Schmidt, Dimitris Tsipras, and
  Adrian Vladu.
\newblock Towards deep learning models resistant to adversarial attacks.
\newblock {\em International Conference on Learning Representations}, 2018.

\bibitem{Miyato:ICLR16}
Takeru Miyato, Shin-ichi Maeda, Masanori Koyama, Ken Nakae, and Shin Ishii.
\newblock Distributional smoothing with virtual adversarial training.
\newblock {\em International Conference on Learning Representations}, 2016.

\bibitem{Moosavi:CVPR17}
Seyed-Mohsen Moosavi-Dezfooli, Alhussein Fawzi, Omar Fawzi, and Pascal
  Frossard.
\newblock Universal adversarial perturbations.
\newblock In {\em {IEEE} Conference on Computer Vision and Pattern
  Recognition}, pages 1765--1773, 2017.

\bibitem{Dezfooli:ICLR18}
Seyed-Mohsen Moosavi-Dezfooli, Alhussein Fawzi, Omar Fawzi, Pascal Frossard,
  and Stefano Soatto.
\newblock Robustness of classifiers to universal perturbations: A geometric
  perspective.
\newblock {\em International Conference on Learning Representations}, 2018.

\bibitem{Moosavi:SP16}
Seyed-Mohsen Moosavi-Dezfooli, Alhussein Fawzi, and Pascal Frossard.
\newblock Deepfool: a simple and accurate method to fool deep neural networks.
\newblock In {\em {IEEE} Conference on Computer Vision and Pattern
  Recognition}, pages 2574--2582, 2016.

\bibitem{Papernot:Arxiv16}
Nicolas Papernot, Patrick McDaniel, and Ian Goodfellow.
\newblock Transferability in machine learning: from phenomena to black-box
  attacks using adversarial samples.
\newblock {\em arXiv preprint arXiv:1605.07277}, 2016.

\bibitem{Papernot:SP16}
Nicolas Papernot, Patrick McDaniel, Somesh Jha, Matt Fredrikson, Z~Berkay
  Celik, and Ananthram Swami.
\newblock The limitations of deep learning in adversarial settings.
\newblock In {\em {IEEE} European Symposium on Security and Privacy
  (EuroS\&P)}, pages 372--387. IEEE, 2016.

\bibitem{Qin:NIPS19}
Chongli Qin, James Martens, Sven Gowal, Dilip Krishnan, Krishnamurthy
  Dvijotham, Alhussein Fawzi, Soham De, Robert Stanforth, and Pushmeet Kohli.
\newblock Adversarial robustness through local linearization.
\newblock In {\em Advances in Neural Information Processing Systems}, pages
  13847--13856, 2019.

\bibitem{Sengupta:GameSec19}
Sailik Sengupta, Tathagata Chakraborti, and Subbarao Kambhampati.
\newblock Mtdeep: Boosting the security of deep neural nets against adversarial
  attacks with moving target defense.
\newblock In {\em International Conference on Decision and Game Theory for
  Security}, pages 479--491. Springer, 2019.

\bibitem{Song:Arxiv17}
Yang Song, Taesup Kim, Sebastian Nowozin, Stefano Ermon, and Nate Kushman.
\newblock Pixeldefend: Leveraging generative models to understand and defend
  against adversarial examples.
\newblock {\em arXiv preprint arXiv:1710.10766}, 2017.

\bibitem{Su:EC19}
Jiawei Su, Danilo~Vasconcellos Vargas, and Kouichi Sakurai.
\newblock One pixel attack for fooling deep neural networks.
\newblock {\em IEEE Transactions on Evolutionary Computation}, 23(5):828--841,
  2019.

\bibitem{Tsuzuku:NIPS18}
Yusuke Tsuzuku, Issei Sato, and Masashi Sugiyama.
\newblock Lipschitz-margin training: Scalable certification of perturbation
  invariance for deep neural networks.
\newblock In {\em Advances in Neural Information Processing Systems}, pages
  6541--6550, 2018.

\bibitem{Warde-Farley:POS16}
David Warde-Farley and Ian Goodfellow.
\newblock Adversarial perturbations of deep neural networks.
\newblock {\em Perturbations, Optimization, and Statistics}, 311, 2016.

\bibitem{Wong:Arxiv17}
Eric Wong and J~Zico Kolter.
\newblock Provable defenses against adversarial examples via the convex outer
  adversarial polytope.
\newblock {\em arXiv preprint arXiv:1711.00851}, 2017.

\bibitem{Xu:NDSS18}
Weilin Xu, David Evans, and Yanjun Qi.
\newblock Feature squeezing: Detecting adversarial examples in deep neural
  networks.
\newblock {\em Network and Distributed Systems Security Symposium}, 2018.

\bibitem{Yasodharan:NIPS19}
Sarath Yasodharan and Patrick Loiseau.
\newblock Nonzero-sum adversarial hypothesis testing games.
\newblock In {\em Neural Information Processing Systems}, pages 7310--7320,
  2019.

\bibitem{Zantedeschi:WAIS17}
Valentina Zantedeschi, Maria-Irina Nicolae, and Ambrish Rawat.
\newblock Efficient defenses against adversarial attacks.
\newblock In {\em {ACM} Workshop on Artificial Intelligence and Security},
  pages 39--49, 2017.

\bibitem{Zhou:KDD12}
Yan Zhou, Murat Kantarcioglu, Bhavani Thuraisingham, and Bowei Xi.
\newblock Adversarial support vector machine learning.
\newblock In {\em {ACM SIGKDD} International Conference on Knowledge Discovery
  and Data Mining}, pages 1059--1067, 2012.

\end{thebibliography}

\newpage
\appendix
\section{Proof of \cref{lem:geometry}}
\begin{figure}[h]
\centering
\includegraphics[height=4cm]{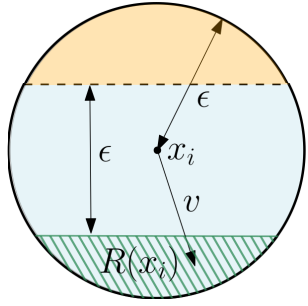}
\caption{\label{fig:robust-set}
Geometry of the Robust Set $R(x_i)$ is shown by the shaded region. The linear approximation around $x_i$, i.e.  $f_L$, is shown by the dotted line. The orange and the blue regions are the two classification regions defined by $f_L$. The perturbation budget is $\epsilon$.}
\end{figure}
\geometry*
\begin{proof} Recall from \cref{def:robust-set} that the robust set is the set of all directions $v_d$ that the defender $D$ can play at a point $x_i \in \X$ such that no matter what (deterministic) action $v_a$ the attacker $A$ plays, it will always have the utility $-1$, i.e. $u_D(x_i, v_d, v_a) = -1$ for all $v_a \in V$:
\begin{equation*}
R(x_i) = \{ v \colon v \in V \ s.t. \  \forall v' \in V \ u_A(x_i, v', v) = -1\}
\end{equation*}

Observe that whenever $x_i + v_d$ lies at a distance more than $\epsilon$ from the linear approximation $f_L$ around $x_i$, we have $sgn(f_L(x_i + v_d + v)) = sgn(f_L(x_i + v_d)) \ \forall v \in B(0, \epsilon)$. Hence, no matter what direction $v_a$ the attacker plays, she always gets a utility of $-1$. This shows that $\text{dist}(x_i + v_d, L) \geq \epsilon$ is a sufficient condition for $v_d \in R(x_i)$ given that $v_d$ lies on the same side of $f_L$ as $x_i$, i.e. $\text{sgn} f_L(x_i + v_d) = \text{sgn} f_L(x_i)$. If $v_d$ lies on the opposite side of $f_L$ as $x_i$, then $v_d \not \in R(x_i)$ trivially as the attacker can play $v_a = 0$ to get $\text{sgn} f_L(x_i + v_d + v_a) \neq \text{sgn} f_L(x_i)$ and thus $+1$ utility. 

The above paragraph showing sufficiency of $\text{dist}(x_i + v_d, L) \geq \epsilon$ does not need any restriction on the decision boundary. However, the locally linear model additionally gives us the necessity of the distance condition, as for any $v_d$ played by the defender with $\text{dist}(x_i + v_d, L) < \epsilon$, the attacker can take $v_a$ to be the FGM direction (i.e. perpendicular to $f_L$ towards the other side of the decision boundary as $x + v_d$) to obtain $\text{sgn} f_L(x_i + v_d + v_a) \neq \text{sgn} f_L(x_i)$, and thus get a $+1$ utility. This shows that $\text{dist}(x_i + v_d, L) \geq \epsilon$ is a necessary condition for $v_d \in R(x_i)$. We have thus shown that the following condition is neccessary and sufficient for $v_d$ to belong to the robust set $R(x_i)$:
\begin{center}
(1) The perturbed point staying at least $\epsilon$ away from the boundary, i.e. $\text{dist}(x_i + v, L) \geq \epsilon$ AND \\
(2) the label staying unchanged, i.e. $\text{sgn} f_L(x_i + v) = \text{sgn} f_L(x_i)$
\end{center}
The geometry of the problem is shown in \cref{fig:robust-set}. As $\text{dist}(x_i + v, L) = \frac{|f_L(x_i + v)|}{\|\nabla f(x_i)\|_2}$, the first condition gives us $|f_L(x_i + v)| - \epsilon \|\nabla f(x_i)\|_2 \geq 0$. The second condition gives us $|f_L(x_i + v)| = \text{sgn}(f(x_i)) f_L(x_i + v)$. Since $f_L(x_i + v) = f(x_i) + \nabla f(x_i)^\top v$, we get the equivalent condition:
\begin{equation*}
\text{sgn}(f(x_i))(f(x_i) + \nabla f(x_i)^\top v) - \epsilon \|\nabla f(x_i)\| \geq 0 \qedhere
\end{equation*}
\end{proof}

The above proof can also be expressed in short by tailoring all parts to our locally linear approximation:
\begin{align*}
R(x) &= \{ v \in B(0, \epsilon) \colon sgn(f(x)) (f(x) + \nabla f(x)^\top (v + v_a)) \geq 0 \ \forall v_a \in B(0, \epsilon) \} \\
&= \{ v \in B(0, \epsilon) \colon \min_{v_a \in B(0, \epsilon)} sgn(f(x)) (f(x) + \nabla f(x)^\top (v + v_a)) \geq 0 \} \\
&= \{v \in B(0, \epsilon) \colon sgn(f(x)) (f(x) + \nabla f(x)^\top v)  - \epsilon \|\nabla f(x)\| \geq 0\}
\end{align*}

\section{Proof of \cref{lem:fgsm-br-cont}}
\fgsm*
\begin{proof}
Let the (possibly randomized) strategies played by the attacker $A$ and the defender $D$ be $s_A$ and $s_D$ respectively. The utility obtained by the attacker is $\bar u_A(s_A, s_D)$: 
\begin{align}
\bar u_A(s_A, s_D) &= \E_{x \sim p_X, a \sim s_A, d \sim s_D} u_A(x, a(x), d(x)) \\
&= \E_{x \sim p_X, d \sim s_D} \E_{a \sim s_A} \Big[u_A(x, a(x), d(x)) \Big| x, d \Big] \\
&= \int_X \int_{\A_D} \left( \int_{\A_A} u_A(x, a(x), d(x)) \ p(a) \ da \right) p(d) \ p(x) \ dd \ dx \label{eq:full-integral} \\
&\leq \E_{x \sim p_X, d \sim s_D} \sup_{a \in \A_A} u_A(x, a(x), d(x)) \quad \text{(Property of convex combination)} \label{eq:upper-bd-cont}
\end{align}
In the above, we have used the fact that $p(a, d, x)$ factorizes as $p(a) p(d) p(x)$, due to the way our game is played. Recall that for each $x \in X$, we are taking the approximate decision boundary to be the zero-contour of a linear approximation of $f$ around $x$, i.e. $f_L(x') = f(x) + \nabla f(x)^\top (x' - x)$. Additionally, by the definition of $R(x)$, we have the following for all $x \in X, d \in \PP(\A_D)$:
\begin{equation} \label{eq:sup-property}
\sup_{a \in \A_A} u_A(x, a(x), d(x)) = 
	\begin{cases}
		-1 \quad \text{ if } d(x) \in R(x) \\
		+1 \quad \text{ otherwise }
	\end{cases}
\end{equation}
Taking the underlying sample-space to be $\Omega = X \times \A_D$, and the associated joint probability distribution over this space to be $p_X \times s_D$, we define the event $E = \{(x, d) \colon (x, d) \in \Omega, \ d(x) \in R(x)\}$. $\bar E$ is defined to be the complement of $E$. From \cref{eq:sup-property}, we see that:
\begin{align}
\E_{x \sim p_X, d \sim s_D} \sup_{a \in \A_A} u_A(x, a(x), d(x))
&= \Pr [\bar E] - \Pr [E]
\end{align}
Now, we will use simple geometry to see that the FGM direction always achieves the upper-bound obtained in \cref{eq:upper-bd-cont}. Recall that the FGM strategy is a deterministic strategy, which plays the funtion $a_\text{FGM}$ with probability 1 such that the distribution induced on $(X, V)$ has its entire mass on $a_\text{FGM}(x) = -\epsilon \frac{{\rm sgn} (f(x))}{\| \nabla f(x) \|_2} \nabla f (x)$ for all $x \in X$. Following the same steps as above till \cref{eq:full-integral}, we have:
\begin{align}
\bar u_A(s_\text{FGM}, s_D) 
&= \int_X \int_{\A_D} \left( \int_{\A_A} u_A(x, a(x), d(x)) \ p(a | d, x) \ da \right) \ p(d) \ p(x) \ dd \ dx \nonumber \\
&= \int_X \int_{\A_D} u_A (x, a_\text{FGM}(x), d(x)) \ p(d) \ p(x) \ dd \ dx \label{eq:integral-fgsm}
\end{align}

When $d(x) \in R(x)$, all attacker directions lead to an utility of $-1$ for the attacker, hence so does the FGM direction, i.e. $u_A (x, a_\text{FGM}(x), d(x)) = -1$. 

When $d(x) \not \in R(x)$, then we claim that $a_\text{FGM}(x)$ is a direction such that ${\rm sgn} f_L(x + d(x) + a_\text{FGM}(x)) \neq {\rm sgn} f_L(x))$. This can be seen by observing that the point closest to $x + d(x)$ on the decision boundary is the first boundary point we hit by moving towards the boundary in a direction perpendicular to it. For the assumed linear boundary $f_L(x) = 0$, this direction is given by $\frac{- {\rm sgn} (f(x))}{\| \nabla f(x) \|_2} \nabla f (x)$. Since $d(x) \not \in R(x)$, there is atleast one vector $v_a \in V$ such that ${\rm sgn} f(x + d(x) + v_a) \neq {\rm sgn} f(x + d(x))$. This implies that one can rotate $v_a$ towards the ray $\{x + k \cdot \frac{- {\rm sgn} (f(x))}{\| \nabla f(x) \|_2} \nabla f (x) \colon k \geq 0\}$ to maintain the sign difference. Since the vector obtained on completing this rotation is exactly the FGM direction $a_\text{FGM}(x)$, we are done. Hence, we have shown that when $d(x) \not \in R(x)$, then $u_A (x, a_\text{FGM}(x), d(x)) = +1$. 

Continuing from \cref{eq:integral-fgsm}, we have:
\begin{align}
&\int_X \int_{\A_D} u_A (x, a_\text{FGM}(x), d(x)) \ p(d) \ p(x) \ dd \ dx \nonumber \\
&= \int_X \int_{\A_D} \Big((+1) \cdot \mathbb{I}[d(x) \in R(x)]  + (-1) \cdot \mathbb{I}[d(x) \not \in R(x)] \Big) \ p(d) \ p(x) \ dd \ dx \\
&= \Pr [\bar E] - \Pr [E]
\end{align}

This shows that for all strategies $s_A \in \PP(\A_A), s_D \in \PP(\A_D)$ played by $A$, $D$ respectively, we have:
\[\bar u_A(s_{\text{FGM}}, s_D) \geq \bar u_A(s_A, s_D) \qedhere\]
\end{proof}

\section{Proof of \cref{lem:smoothing-br-cont}}
\smooth*
\begin{proof}
Let $s_D$ be the strategy followed by $D$, specified by the distribution $p_D \in \PP(\A_D)$. Recall that the attackers strategy $s_\text{FGM}$ plays the function $a_\text{FGM}$ with probability $1$ such that the distribution induced on $(X, V)$ has its entire mass on $a_\text{FGM}(x) = -\epsilon \frac{{\rm sgn} (f(x))}{\| \nabla f(x) \|_2} \nabla f (x)$ for all $x \in X$. The defender's utility can be written as follows:
\begin{align}
\bar u_D(s_\text{FGM}, s_D) 
&= \E_{x \sim p_X, d \sim p_D} u_D(x, a_\text{FGM}(x), d(x)) \\
&= \int_X \int_{\A_D} u_D(x, a_\text{FGM}(x), d(x)) \ p_D(d) \ p_X(x) \ dd \  dx \label{eq:improve-cont}
\end{align}

Following the proof of \cref{lem:fgsm-br-cont}, we can see that under the FGM attack, the defender gets a utility of $+1$ at the point $x \in X$ when he plays from the robust-set $R(x)$, i.e. for a sample $d \sim p_D$ we have $d(x) \in R(x)$, and a utility of $-1$ otherwise. Accordingly, we now first split the domain in \cref{eq:improve-cont} into two parts depending on whether the defender plays a direction in the robust-set:
\begin{align}
\bar u_D(s_\text{FGM}, s_D) 
&= \int_X \int_{\A_D} u_D(x, a_\text{FGM}(x), d(x)) \mathbb{I}[d(x) \in R(x)] \ p_D(d) \ p_X(x) \ dd \  dx + \nonumber \\
&\quad \int_X \int_{\A_D} u_D(x, a_\text{FGM}(x), d(x)) \mathbb{I}[d(x) \not \in R(x)] \ p_D(d) \ p_X(x) \ dd \  dx \\
&= \int_X \int_{\A_D} (+1) \mathbb{I}[d(x) \in R(x)] \ p_D(d) \ p_X(x) \ dd \  dx + \nonumber \\
&\quad \int_X \int_{\A_D} (-1) \mathbb{I}[d(x) \not \in R(x)] \ p_D(d) \ p_X(x) \ dd \  dx \\
&= \int_{\A_D} \int_X (+1) \mathbb{I}[d(x) \in R(x)] \ p_D(d) \ p_X(x) \ dx \  dd + \nonumber \\
&\quad \int_{\A_D} \int_X (-1) \mathbb{I}[d(x) \not \in R(x)] \ p_D(d) \ p_X(x) \ dx \  dd \label{eq:interchange-cont}
\end{align}

The order of integration could be interchanged in \cref{eq:interchange-cont} since the double integral of the absolute value of the integrand is finite (Fubini's Theorem). We now appeal to the structure of $\A_D$, and recall that $\A_D$ consists of all constant functions from $X$ to $V$. For a particular element $d \in \A_D$, let $v_d$ be its output such that $\forall x \in X \ d(x) = v_d$. Further, we note from the definition of $\phi(v)$ that $\int_X \mathbb{I}[v \not \in R(x)] p_X(x) dx = 1 - \phi(v)$. Continuing from \cref{eq:interchange-cont}:
\begin{align}
\bar u_D(s_\text{FGM}, s_D) 
&= \int_{\A_D} \phi(v_d) \ p_D(d) \ dd - \int_{\A_D} \int_X \mathbb{I}[d(x) \not \in R(x)] \ p_D(d) \ p_X(x) \ dx \  dd \\
&= \int_{\A_D} \phi(v_d) \ p_D(d) \ dd - \int_{\A_D} (1 - \phi(v_d)) \ p_D(d) \ dd \\
&= \int_{\A_D} (2 \phi(v_d) - 1) \ p_D(d) \ dd \label{eq:defense-cont} \\
&\leq 2 \phi(v^*) - 1 \quad \text{ (By property of convex combination for any $v^* \in V^*$) } \label{eq:opt-defense-cont}
\end{align}
Finally, we observe that $s_\text{SMOOTH}$ achieves the upper-bound obtained in \cref{eq:opt-defense-cont}. Let $p$ be the density for the uniform distribution over the set $F^*$. Following the same steps as above till \cref{eq:opt-defense-cont}, we will get the following:
\begin{align}
\bar u_D(s_\text{FGM}, s_\text{SMOOTH}) 
&= \int_{\A_D} (2 \phi(v_d) - 1) \ p(d) \ dd \\
&= \int_{F^*} (2 \phi(v_d) - 1) \ p(d) \ dd \\
&= (2 \phi(v^*) - 1) \int_{F^*} p(d) \ dd \\
&= (2 \phi(v^*) - 1)
\end{align}

Hence, we have have shown that $u_D(s_\text{FGM}, s_\text{SMOOTH}) \geq u_D(s_\text{FGM}, s_D)$ for all $s_D \in \PP(\A_D)$. 
\end{proof}

\section{Details for Experiments}

\begin{table}[h]
  \caption{Attacks and Defenses for MNIST $0$ vs $1$. The FGM attack and SMOOTH defense correspond to $s_\text{FGM}$ and $s_\text{SMOOTH}$ respectively. The PGD attack \cite{Madry:ICLR18} is an iterated version of FGM.
True Accuracy shows accuracies using the true classifier $f$ and Approximate Accuracy shows accuracies according to the locally linear approximation $f_L$. Detailed descriptions can be found in the Appendix.}
  \label{tab:resultsapp}
  \centering
  \begin{tabular}{@{\,}c@{\;\;}c@{\;\;}c@{\;\;}c@{\,}}
    \toprule
    Attack & Defense & True Accuracy (\%) & Approximate Accuracy (\%) \\
    \midrule
    - & - & 99.9 & 99.9 \\
    FGM & - & 63.0 & 48.3\\
    FGM & SMOOTH & 95.6 & 94.5\\
    PGD & - & 47.7 & 85.6 \\
    PGD & SMOOTH & 75.1 & 99.1\\
    \bottomrule
  \end{tabular}
\end{table}
\cref{tab:resultsapp} shows results for a particular binary classification task ($0$ vs $1$) on the MNIST dataset. The \cref{tab:results} in the main text reports summary statistics for this table over all possible pairs on MNIST and FMNIST. A particular peculiarity is that the network when attacked with PGD has a better \emph{approximate accuracy} than the network attacked with FGM, whereas we know that PGD is a stronger attack than FGM. This happens due to the fact that approximate accuracy is defined on a linearized model (see description in \cref{tab:resultsappext}), whereas PGD observes gradients for the original model, leading to a bad attack performance when the evaluation is made according to the linearized model.

\begin{table}[h]
\caption{Attacks and Defenses for MNIST $0$ vs $1$. This is an expanded version of \cref{tab:resultsapp}. \label{tab:resultsappext}}
  \centering
  \begin{tabular}{Sc Sc Sc Sc Sc}
    \toprule
    Attack & Defense & Accuracy (\%) & Description & \\
    \midrule
    - & - & 99.9 & $\frac{1}{n} \sum_{i = 1}^{n} \mathbf{1}[sgn f(x_i) = y_i]$ & \multirow{5}{*}{True} \\
    FGM & - & 63.0 & $\frac{1}{n} \sum_{i = 1}^{n} \mathbf{1}[sgn f(x_i + v_a) = y_i]$ & \\
    FGM & SMOOTH & 95.6 & $\frac{1}{n} \sum_{i = 1}^{n} \mathbf{1}[sgn f(x_i + a_\text{FGM}(x_i) + v_n^*) = y_i]$ & \\
    PGD & - & 47.7 & $\frac{1}{n} \sum_{i = 1}^{n} \mathbf{1}[sgn f(x_i + a_\text{PGD}(x_i)) = y_i]$ & \\
    PGD & SMOOTH & 75.1 & $\frac{1}{n} \sum_{i = 1}^{n} \mathbf{1}[sgn f(x_i + a_\text{PGD}(x_i) + v_n^*) = y_i]$ & \\
    \midrule
    - & - & 99.9 & $\frac{1}{n} \sum_{i = 1}^{n} \mathbf{1}[sgn f_L(x_i) = y_i]$ & \multirow{5}{*}{Approximate} \\
    FGM & - & 48.3 & $\frac{1}{n} \sum_{i = 1}^{n} \mathbf{1}[sgn f_L(x_i + a_\text{FGM}(x_i)) = f_L(x_i)]$ & \\
    FGM & SMOOTH & 94.5 & $\frac{1}{n} \sum_{i = 1}^{n} \mathbf{1}[sgn f_L(x_i + a_\text{FGM}(x_i) + v_n^*) = f_L(x_i)]$ & \\
    PGD & - & 85.6 & $\frac{1}{n} \sum_{i = 1}^{n} \mathbf{1}[sgn f_L(x_i + a_\text{PGD}(x_i)) = f_L(x_i)]$ & \\
    PGD & SMOOTH & 99.1 & $\frac{1}{n} \sum_{i = 1}^{n} \mathbf{1}[sgn f_L(x_i + a_\text{PGD}(x_i) + v_n^*) = f_L(x_i)]$ & \\
    \bottomrule
\end{tabular}
\end{table}

\cref{tab:resultsappext} is an expanded version of \cref{tab:resultsapp},  where the Description column has been added which shows how the accuracies have been computed. $v_n^*$ is obtained by following the optimization procedure mentioned in \cref{sec:approximation}. $a_\text{PGD}(x_i)$ is obtained by repeatedly applying FGM and projecting to the set of allowed perturbations $V$, i.e. we iterate the following steps $10$ times for each test point $x_i$ to obtain $p_1, p_2, \ldots, p_{10}$, starting with $p_0 = x_i$:
\begin{enumerate}
\item Perturb current iterate $p_{j}$: $p'_j \gets p_{j} + a_\text{FGM}(p_{j})$
\item Project perturbation to $B(x_i, \epsilon)$: $p_{j + 1} \gets x_i + \epsilon \frac{p'_j - x_i}{\|p'_j - x_i\|}$
\end{enumerate}
At the end we get $a_\text{PGD}(x_i) = p_{10} - x_i$.

\section{Validity of Modelling Assumptions}
While the assumption of local-linearity might seem strong at first, there is ample empirical evidence of its validity for neural networks. Fig.~\ref{fig:church-plot}, reproduced from \cite{Warde-Farley:POS16} shows the
decision boundaries of a CNN trained on CIFAR-10 in a $\epsilon$ neighbourhood of many randomly-selected images, where white denotes the predicted class and other shades denote other classes. It can be seen that, locally, the boundary is approximately linear. 
This \emph{linearity hypothesis} was proposed in \cite{Goodfellow:ICLR15} and further explored in \cite{Warde-Farley:POS16} (showing empirical evidence) and \cite{Dezfooli:ICLR18} (linking to the existence of universal adversarial perturbations). Recent studies \cite{Lee:Arxiv19}, \cite{Qin:NIPS19} improve Deep Neural Networks' robustness by promoting local-linearity. Hence, we stress that our modelling assumptions \emph{do} partially hold for modern real-world classifiers, and we are not limited to  just linear classifiers.

\begin{figure}[h]
\centering
\includegraphics[width=0.5\linewidth]{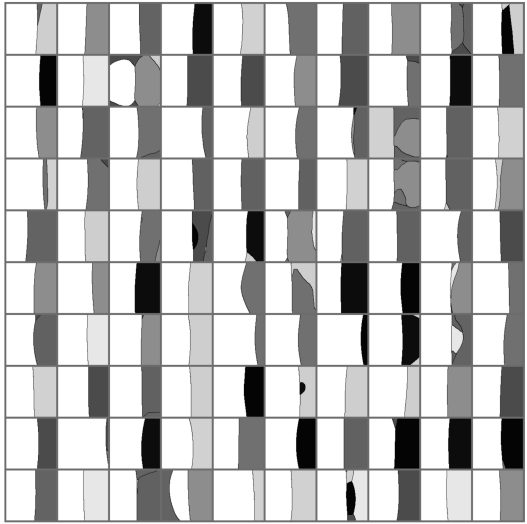}
\caption{Church-Window plots for a CNN $f$ reproduced from Fig.~11.2 of \cite{Warde-Farley:POS16}. Each plot shows $f(\mathbf{x} + a\mathbf{u} + b\mathbf{v})$ for $a, b \in [-\epsilon, \epsilon]$, where $\mathbf{u}$ is the FGM direction, $\mathbf{v}$ is a random direction orthogonal to $\mathbf{u}$ and $\mathbf{x}$ is a random data-point from CIFAR-10. White denotes the class $f(\mathbf{x})$, and other shades denote other classes.}
\label{fig:church-plot}
\end{figure}

\end{document}